%% file: LaplacianSmooth_Final.tex
\documentclass[journal, twocolumn, 10pt]{IEEEtran}
\usepackage{url}
\usepackage{graphicx}
\usepackage{epstopdf}
\usepackage{algorithm}
\usepackage{algorithmicx}
\usepackage{algpseudocode}
\usepackage{cite}
\usepackage{color}
\usepackage{amsmath,amssymb,amsfonts,amsthm}
\usepackage{bm}
\usepackage{setspace}
\usepackage{multirow}
\usepackage{mathtools}
\usepackage{wrapfig}
\usepackage{subcaption}
\usepackage{multicol}
\usepackage{lipsum}
\usepackage{fancyhdr}

\newtheorem{thm} {Theorem}

\newtheorem{cor} {Corollary}

\include{notation_20161122}

\begin{document}

\title{Bias-Variance Tradeoff of Graph Laplacian Regularizer
	}

\author{Pin-Yu~Chen and Sijia Liu
	\thanks{P.-Y. Chen is with AI Foudations, IBM Thomas J. Watson Research Center, Yorktown Heights, NY 10598, USA. Email : pin-yu.chen@ibm.com.
		S. Liu is with the Department of Electrical Engineering and Computer Science, University of Michigan, Ann Arbor, MI 48109, USA. Email : lsjxjtu@umich.edu.}
}

\maketitle
\thispagestyle{empty}
\begin{abstract}
This paper presents a bias-variance tradeoff of graph Laplacian regularizer, which is widely used in graph signal processing and semi-supervised learning tasks. The scaling law of the optimal regularization parameter is specified in terms of the spectral graph properties and a novel signal-to-noise ratio parameter, which suggests selecting a mediocre regularization parameter is often suboptimal. The analysis is applied to three applications, including random, band-limited, and multiple-sampled graph signals. Experiments on synthetic and real-world graphs demonstrate near-optimal performance of the established analysis.
\end{abstract}

\begin{IEEEkeywords}
	graph signal processing, mean squared error analysis, scaling law, spectral graph theory
\end{IEEEkeywords}

\section{Introduction}

Graph Laplacian regularizer (GLR) has been widely used in graph signal processing, semi-supervised learning and image filtering tasks \cite{Shuman13,Bertrand13Mag,milanfar2013tour,Sandryhaila14,wu2016estimating}. Regularization techniques involving the graph Laplacian method can be interpreted in different perspectives. In a regression setting, GLR penalizes incoherent (i.e., non-smooth) signals across adjacent nodes \cite{belkin2004regularization,belkin2004semi,belkin2006manifold,Anis16,chepuri2016learning,Onuki16}. In a probability model setting, GLR is used as a prior distribution that favors smooth signals \cite{Shuman13,CPY14deep,chen2015discrete,CPY14spectral,Dong16_Laplacian_Learning,kalofolias2016learn,chen2016signal,CPY17amos_ICASSP,Liu17smooth}. 

This paper presents a bias-variance tradeoff of GLR. In particular, the scaling law of the optimal regularization parameter of GLR that balances the bias-variance tradeoff is specified in terms of the spectral graph properties and a novel signal-to-noise ratio (SNR) parameter. Our analysis shows an abrupt change in the order of the optimal regularization parameter when varying the SNR parameter,  suggesting that selecting a mediocre regularization parameter is often suboptimal,  which provides novel insights in the analysis and utility of GLR. 
We then apply the bias-variance tradeoff analysis to random, band-limited, and multiple-sampled graph signals, and specify the SNR parameter for each case. Experiments on synthetic and real-world graphs verify the scaling law analysis and demonstrate near-optimal performance in terms of the mean squared error. The proofs of the established theoretical results are given in the appendices of the supplementary material.

Consider a weighted undirected  connected simple graph $\cG(\cV,\cE)$ of $n$ nodes and $m$ edges, where $\cV$ ($\cE$) is the set of nodes (edges). The weight of an edge $(i,j) \in \cE$ is specified by the entry $W_{ij}>0$ of an $n \times n$ symmetric matrix $\bW$. The graph Laplacian matrix of $\cG$ is defined as $\bL=\bS-\bW$, where $\bS=\diag(\bW \bone_n)$ is a diagonal matrix, and $\bone_n$ is the $n \times 1$ column vector of ones. Let $(\lambda_i,\bv_i)$, $1 \leq i \leq n$, denote the $i$-th smallest eigenpair of $\bL$ such that its eigenvalue decomposition can be written as $\bL=\sum_{i=1}^n \lambda_i \bv_i \bv_i^T$, where $\{\lambda_i\}_{i=1}^n$ is a non-decreasing sequence, and $\bv_i^T \bv_j=1$ if $i=j$ and  $\bv_i^T \bv_j=0$ if $i \neq j$. For a connected graph $\cG$, it is well-known from spectral graph theory \cite{Chung97SpectralGraph} that $(\lambda_1,\bv_1)=(0,\bone)$,  where $\bone= \frac{\bone_n}{\sqrt{n}}$, and $\lambda_i > 0$ for all $2 \leq i \leq n$. Another useful property that leads to the smoothing effect is that for any vector $\bx \in \bbR^n$, 
\begin{align}
\label{eqn_Laplacian_quadratic}
\bx^T \bL \bx = \sum_{(i,j)\in \cE} W_{ij} (x_i - x_j)^2,
\end{align}
where $x_i$ is the $i$-th entry of $\bx$. We also call (\ref{eqn_Laplacian_quadratic}) the GLR.
\section{Bias-Variance Tradeoff Analysis}
\label{sec_bias_var}
Let $\by \in \bR^n$ be a vector of observed signals from the graph $\cG$, where its entry $y_i$ corresponds to the observed signal on node $i$. 
Assume an additive noise model $\by=\bx^*+\be$, where $\bx^* \in \bR^n$ is the unknown ground-truth signal
and $\be  \in \bR^n $ is the vector accounting for random errors on each node, where  $\be$ has zero mean and covariance structure $\bSigma$, which is different from the assumption of additive Gaussian noise in image filtering, such as the SURE estimator. \cite{ramani2008monte,milanfar2013tour}.
For many signal processing and semi-supervised learning tasks, given a noisy graph signal $\by$ on $\cG$, one aims to recover a smooth graph signal. This can be casted as a least-square minimization problem regularized by the GLR \cite{Shuman13,belkin2004regularization,Dong16_Laplacian_Learning},
\begin{align}
\label{eqn_GLSR_regression}
\min_{\bx \in \bbR^n}  \|\by- \bx\|_2^2 + \alpha \bx^T \bL \bx ,
\end{align}
where  $\|\cdot\|_2$ denotes the Euclidean distance, and $\alpha \geq 0$ is the regularization parameter. In essence, one is interested in obtaining a solution $\bxhat$ to (\ref{eqn_GLSR_regression}) such that with a proper selection of the regularization parameter $\alpha$, the vector $\bxhat$ is smooth in the sense that the weighted sum of squared signal difference of all adjacent node pairs in (\ref{eqn_Laplacian_quadratic}) is confined. 
What remains unclear is the effect of $\alpha$ on the estimator $\bxhat$, which is the main contribution (optimal scaling law analysis) of this paper.

It is easy to show that $\bxhat$ has an analytical expression
\begin{align}
\label{eqn_MSE_solution}
\bxhat=(\bI+\alpha \bL)^{-1} \by =: \bH \by,
\end{align}
where the eigenvalue decomposition of $\bH$ can be written as 
\begin{align}
\label{eqn_H}
\bH=(\bI+\alpha \bL)^{-1}=\sum_{i=1}^n \frac{1}{1+ \alpha \lambda_i} \bv_i \bv_i^T=:\sum_{i=1}^n h_i \bv_i \bv_i^T.
\end{align}
In particular, $h_1=1$ since $\lambda_1=0$. 

For a fixed $\alpha$, the bias of $\bxhat$ is 
\begin{align}
\label{eqn_bias}
\textnormal{Bias($\alpha$)}=\|\bbE \bxhat - \bx^*\|_2=\|(\bH-\bI) \bx^*\|_2,
\end{align}
where $\bI$ is the identity matrix.
The variance of $\bxhat$ is 
\begin{align}
\label{eqn_variance}
\textnormal{Var($\alpha$)}=\trace (\cov(\bxhat))=\trace(\bH^2 \bSigma),
\end{align}
where $\cov(\bxhat)$ denotes the covariance matrix of $\bxhat$.
As a result, the mean squared error (MSE) can be expressed as 
\begin{align}
\label{eqn_MSE}
\textnormal{MSE($\alpha$)}=\bbE \|\bxhat - \bx^*\|_2^2=\textnormal{Bias($\alpha$)}^2 + \textnormal{Var($\alpha$)}.
\end{align}

The following theorem shows that using GLR decreases the variance of the estimator $\bxhat$ when compared to the case of without using GLR (i.e., $\alpha=0$). 
\begin{thm}
	\label{thm_var}
For any $\alpha>0$, \textnormal{Var($\alpha$)} $\leq$ \textnormal{Var($0$)}. The inequality becomes strict if $\Sigma$ has full rank. 
\end{thm}
\begin{proof}
	The proof is given in Appendix \ref{proof_var}.
\end{proof}

Theorem \ref{thm_var} suggests that selecting any  $\alpha>0$ can decrease the variance. However, the selection of $\alpha$ also affects the bias in (\ref{eqn_bias}), which is known as the bias-variance tradeoff.
The analysis below provides the optimal order of $\alpha$ that balances the bias-variance tradeoff.
Applying the Von Neumann's trace inequality \cite{HornMatrixAnalysis} to the variance term in (\ref{eqn_variance}), we have $\textnormal{Var($\alpha$)}=\trace(\bH^2 \bSigma) \leq \sum_{i=1}^n h_i^2 \phi_i$, where $\phi_i$ is the $i$-th largest eigenvalue of $\bSigma$, and the equality holds when $\bSigma$ is a diagonal matrix.
To simplify our analysis, in the rest of this paper we assume
$\bSigma=\diag(\bsigma)$, where $\bsigma=[\sigma_1^2, \sigma_2^2, \ldots, \sigma_n^2]$ and $\sigma_i \geq 0$ denotes the standard deviation. The bias-variance tradeoff for the case of non-diagonal covariance structure can be analyzed in a similar way.
Upon defining $\bQ=\bI-\bH$, it is known from (\ref{eqn_H}) that the eigenvalue decomposition of $\bQ$ is $\bQ=\sum_{i=2}^n \frac{1}{1+\frac{1}{\alpha \lambda_i}} \bv_i \bv_i^T=:\sum_{i=2}^n q_i \bv_i \bv_i^T$  for any $\alpha>0$.
\begin{thm}
	\label{thm_MSE}
	If $\Sigma=\diag(\bsigma)$, then for any $\alpha>0$,
$\textnormal{Bias($\alpha$)$^2$}=\sum_{i=2}^n q_i^2 (\bv_i^T \overline{\bx^*} ) ^2$,	
$\textnormal{Var($\alpha$)}= \sum_{i=1}^n h_i^2 \sigma_i^2$, and therefore
\begin{align}
	\textnormal{MSE($\alpha$)}=\sum_{i=2}^n q_i^2 (\bv_i^T \overline{\bx^*} ) ^2 + \sum_{i=1}^n h_i^2 \sigma_i^2, 
\end{align}
 where $\overline{\bx^*}=\bx^*-\frac{\bone_n^T\bx^*}{n} \bone_n$, $q_i=\frac{1}{1+\frac{1}{\alpha \lambda_i}}$, and $h_i=\frac{1}{1+ \alpha \lambda_i}$.
\end{thm}
\begin{proof}
	The proof is given in  Appendix \ref{proof_MSE}.
\end{proof}
Recall that $h_1=1$ from (\ref{eqn_H}). Theorem \ref{thm_MSE} indicates that there is an universal lower bound $\textnormal{MSE($\alpha$)} \geq \sigma_1^2$ for any $\alpha>0$. 
Theorem \ref{thm_MSE} also implies a clear bias-variance tradeoff since $q_i=1-h_i$ for all $2 \leq i \leq n$. Specifically, increasing $\alpha$ leads to the decrease in variance but also the increase in bias, and vice versa. This tradeoff means that improper selection of $\alpha$ may lead to undesired MSE, as one term will dominate the other. 
The following results provide guidelines on the selection of proper $\alpha$.

\begin{cor}[\textnormal{MSE-UB}]
	\label{cor_MSE_bound}
		If $\Sigma=\diag(\bsigma)$, then for any $\alpha>0$, 
	\begin{align}
	\textnormal{MSE($\alpha$)} 
	&\leq  \lb \frac{1}{1+\frac{1}{\alpha\lambda_n}} \rb^2 \sum_{i=2}^n (\bv_i^T \overline{\bx^*} ) ^2+ \lb \frac{1}{1+\alpha \lambda_2}\rb^2 \sum_{i=2}^n \sigma_i^2 \nonumber \\  
	&~~~+\sigma_1^2, \nonumber
	\end{align} 		
	where the equality holds if $\cG$ is a complete graph of identical edge weight, and the RHS\footnote{RHS means the right hand side.} is denoted by \textnormal{MSE-UB($\alpha$)}.
\end{cor}
\begin{proof}
	The proof is given in Appendix \ref{proof_MSE_bound}.
\end{proof}
MSE-UB in Corollary \ref{cor_MSE_bound} provides a tight upper envelope function for assessing MSE. In Sec. \ref{sec_perform} near-optimal performance of MSE-UB relative to MSE is validated in synthetic and real-world graphs.
Note that since MSE($\alpha$) is  a non-convex function with respect to $\alpha>0$, the optimal $\alpha$ that minimizes MSE does not have a close-form expression. On the other hand, the optimal solution to MSE-UB($\alpha$) can be obtained by solving the roots of a third-order polynomial function, which is the derivative of MSE-UB($\alpha$) with respect to $\alpha$.  
Corollary \ref{cor_MSE_bound} can also be used to specify an optimal value $\alpha^*$ that matches the order of the bias and variance terms appeared in MSE-UB (i.e., the first two terms), which is stated as follows.
\begin{thm}
	\label{thm_match}
	Let $\theta=\sqrt{\frac{\sum_{i=2}^n  \sigma_i^2}{ \sum_{i=2}^n (\bv_i^T \overline{\bx^*} ) ^2}}$. The optimal value that matches the order of the first two terms of MSE-UB($\alpha$) in Corollary \ref{cor_MSE_bound} is 
	\begin{align}
\alpha^* = \frac{(\beta \theta-1) \lambda_n +\sqrt{(\beta\theta-1)^2 \lambda_n^2+4 \lambda_n \lambda_2 \beta\theta}} {2 \lambda_n \lambda_2}, \nonumber
	\end{align}
where $\beta>0$ is some constant such that $\lb \frac{1+\alpha^* \lambda_2}{1+\frac{1}{\alpha^*\lambda_n}} \rb^2= \beta^2  \theta^2 $. 	
\end{thm}
\begin{proof}
	The proof is given in Appendix \ref{proof_match}.
\end{proof}
Theorem \ref{thm_match} suggests that the optimal order-matching regularization parameter for balancing the bias-variance tradeoff depends on the parameter $\theta$ and the eigenvalues $\lambda_2$ and $\lambda_n$ of the graph Laplacian matrix $\bL$ of the graph $\cG$.
Define the effective signal-to-noise ratio to be
\begin{align}
\textnormal{E-SNR}=\frac{ \sum_{i=2}^n (\bv_i^T \overline{\bx^*} ) ^2}{\sum_{i=2}^n  \sigma_i^2}
\end{align}  
such that $\theta=\sqrt{\frac{1}{\textnormal{E-SNR}}}$. The term  $(\bv_i^T \overline{\bx^*} ) ^2$ in \textnormal{E-SNR} is associated with the signal power on graph frequency domain, as $\bv_i^T \overline{\bx^*}=\bv_i^T \bx^*$, for all $2 \leq i \leq n$, where the latter is the corresponding graph Fourier coefficient of $\bx^*$ \cite{Shuman13}. Given that $\bv_1^T \overline{\bx^*}=0$, the term $\sum_{i=2}^n (\bv_i^T \overline{\bx^*} ) ^2 = \sum_{i=1}^n (\bv_i^T \bx^* ) ^2 - \frac{(\bone_n^T \bx^*)^2}{n}$ is the signal power of $\overline{\bx^*}$. The order of $\alpha^*$ in different E-SNR regimes is summarized in the following corollary.
\begin{cor}[scaling law]
\label{cor_SNR}
 Given a graph $\cG$, in the high E-SNR regime $(\theta \ll \frac{1}{\beta})$, $\alpha^* = O \lb \frac{\theta}{\lambda_n } \rb$, in the low E-SNR regime $(\theta \gg \frac{1}{\beta})$, $\alpha^* = O \lb \frac{\theta}{ \lambda_2} \rb$, and in the moderate E-SNR regime $(\theta \approx \frac{1}{\beta})$,  $\alpha^* = O \lb \sqrt{\frac{\theta}{\lambda_n \lambda_2}} \rb$.
\end{cor} 
\begin{proof}
	The proof is given in Appendix \ref{proof_SNR}.
\end{proof}
Corollary \ref{cor_SNR} specifies the scaling law of the order-matching regularization parameter $\alpha^*$ in terms of the parameter $\theta$ (i.e., E-SNR) and the spectral graph properties (i.e., $\lambda_2$ and $\lambda_n$). It also suggests that as E-SNR approaches infinity, $\alpha^*$ will approach $0$. Furthermore, as one sweeps the E-SNR from the high E-SNR regime (small $\theta$) to the low E-SNR regime (large $\theta$), Corollary \ref{cor_SNR} indicates that the order of $\alpha^*$ is expected to have an abrupt boost that depends on the ratio $\frac{\lambda_n}{\lambda_2}$. More importantly,
Corollary \ref{cor_SNR} shows that selecting a mediocre value of the regularization parameter $\alpha$ for GLR is often suboptimal for minimizing the MSE. In the small $\theta$ regime small $\alpha$ is preferred, whereas in the large $\theta$ regime large $\alpha$ is preferred.

\section{Applications to Random, Band-Limited, and Multiple-Sampled Graph Signals}
In this section we apply the bias-variance tradeoff analysis presented in Sec. \ref{sec_bias_var} to  random, band-limited, and multiple-sampled graph signals, respectively. In particular, for each case we specify the parameter $\theta$ governing the order of the optimal order-matching regularization parameter $\alpha^*$.

For graph signals with multiple samples, let $\{\by_t\}_{t=1}^T$ denote the $T$ i.i.d. copies of $\by$ and denote their ensemble average by $\bybar=\frac{\sum_{t=1}^T \by_t}{T}$. By replacing $\by$ in (\ref{eqn_GLSR_regression}) with $\bybar$, the following corollary provides an upper bound on the MSE of i.i.d. multiple-sampled graph signals.
\begin{cor}[Multiple-sampled i.i.d. graph signals]
	\label{cor_multi_signal}
	Let $\{\by_t\}_{t=1}^T$ be $T$ i.i.d. graph signals and let $\bybar=\frac{\sum_{t=1}^T \by_t}{T}$.
	Replacing $\by$ in (\ref{eqn_GLSR_regression}) with $\bybar$,
	if $\Sigma=\diag(\bsigma)$, then for any $\alpha>0$, 
	\begin{align}
	\textnormal{MSE($\alpha$)} 
	&\leq  \lb \frac{1}{1+\frac{1}{\alpha\lambda_n}} \rb^2 \sum_{i=2}^n (\bv_i^T \overline{\bx^*} ) ^2+ \lb \frac{1}{1+\alpha \lambda_2}\rb^2 \frac{\sum_{i=2}^n \sigma_i^2}{T} \nonumber \\  
	&~~~+\sigma_1^2, \nonumber
	\end{align} 	
	where the equality holds if $\cG$ is a complete graph of identical edge weight.
\end{cor}
\begin{proof}
	The proof is given in Appendix \ref{proof_multi_signal}.
\end{proof}
Corollary \ref{cor_multi_signal} shows that for a fixed $\alpha$, the number $T$ of i.i.d. observations has a linear scaling effect (i.e., $\frac{1}{T}$) on the variance term but has no effect on the bias term. Furthermore,  by defining $\theta=\sqrt{\frac{\sum_{i=2}^n  \sigma_i^2}{T \sum_{i=2}^n (\bv_i^T \overline{\bx^*} ) ^2}}$ and applying 
the results in Theorem \ref{thm_match} and Corollary \ref{cor_SNR},
the optimal order of $\alpha^*$ scales with $\frac{1}{\sqrt{T}}$
in the high E-SNR regime and the low E-SNR regime, and scales with $\frac{1}{\sqrt[4]{T}}$ in the moderate E-SNR regime.

For band-limited graph signals, the ground-truth signal $\bx^*$ is a  linear combination of a subset of the basis $\{\bv_i\}_{i=1}^n$ associated with the graph Laplacian matrix $\bL$ \cite{Shuman13,Sandryhaila13,chen2015discrete}, which can be written as $\bx^*=\sum_{j \in \cA} \omega_j\bv_j $, where $\omega_j \neq 0$ and $\cA \subset \{1,2,\ldots,n\}$ indicates the set of active basis from $\{\bv_i\}_{i=1}^n$. The following corollary provides an upper bound on the MSE of band-limited graph signals.
\begin{cor}[Band-limited graph signals]
	\label{cor_band_signal}
	If $\bx^*=\sum_{j \in \cA} \omega_j\bv_j $ and $\Sigma=\diag(\bsigma)$, then for any $\alpha>0$, 
	\begin{align}
	\textnormal{MSE($\alpha$)} 
	&\leq  \lb \frac{1}{1+\frac{1}{\alpha\lambda_n}} \rb^2 \sum_{j \in \cA / \{1\}} \omega_j^2+ \lb \frac{1}{1+\alpha \lambda_2}\rb^2 \sum_{i=2}^n \sigma_i^2 \nonumber \\  
	&~~~+\sigma_1^2, \nonumber
	\end{align}
\end{cor}
\begin{proof}
	The proof is given in Appendix \ref{proof_band_signal}.
\end{proof}
Corollary \ref{cor_band_signal} indicates that the term  $\sum_{j \in \cA / \{1\}} \omega_j^2$ can be viewed as the effective signal strength for band-limited graph signals. Moreover, the coefficient $\omega_1$ corresponding to the coherent basis $\bone$ does not contribute to the MSE. Using the terminology from filter bank design \cite{Shuman13}, GLR is a low-pass filter that excludes the lowest frequency $\omega_1$ in terms of MSE. 
The results in Theorem \ref{thm_match} and Corollary \ref{cor_SNR} can be applied to band-limited graph signals by setting $\theta=\sqrt{\frac{\sum_{i=2}^n  \sigma_i^2}{\sum_{j \in \cA / \{1\}} \omega_j^2}}$.

For random graph signals, assume the ground-truth graph signal $\bx^* \sim \cN(\mu \bone_n,\diag(\bs))$ is a Gaussian random vector with mean $\mu \bone_n$ and covariance $\diag(\bs)$, where $\bs=[s_1^2,s_2^2,\ldots,s_n^2]$ and $s_i \geq 0$.
The following corollary provides an upper bound on \textnormal{MSE-UB($\alpha$)} for random graph signals.
\begin{cor}[Random graph signals]
	\label{cor_random_signal}
	If $\bx^* \sim \cN(\mu \bone_n,\diag(\bs))$ and $\Sigma=\diag(\bsigma)$, let $\overline{s}=\frac{\sum_{i=1}^n s_i^2}{n}$ and $\overline{\sigma}=\frac{\sum_{i=2}^n \sigma_i^2}{n-1}$. Then for any $\alpha>0$, 
	\begin{align}
	\textnormal{MSE($\alpha$)}
	&\leq \lb \frac{1}{1+\frac{1}{\alpha\lambda_n}} \rb^2  (n-1)\overline{s} + \lb \frac{1}{1+\alpha \lambda_2}\rb^2  (n-1) \overline{\sigma} \nonumber \\
	&~~~ + \sigma_1^2 \nonumber,	
	\end{align} 	 
	where the equality holds if $s_i=s \geq 0$ for all $1 \leq i \leq n$.
\end{cor}
\begin{proof}
	The proof is given in Appendix \ref{proof_random_signal}.
\end{proof}
Corollary \ref{cor_random_signal} shows that the mean $\mu \bone_n$ of random graph signals does not contribute to 
the upper bound on \textnormal{MSE($\alpha$)}, which suggests that GLR filters out the mean of the random graph signal in addition to the smoothing effect.
The results  in Theorem \ref{thm_match} and Corollary \ref{cor_SNR} can readily be applied to random graph signals via Corollary \ref{cor_random_signal}.
In particular, define $\theta=\sqrt{ \frac{\overline{\sigma}}{\overline{s}}}$, then the E-SNR becomes $ \frac{\overline{s}}{\overline{\sigma}}$, which is close to the average SNR  $\frac{\overline{\bs}}{\widetilde{\sigma}}$, where $\widetilde{\sigma}=\lb 1-\frac{1}{n}\rb\overline{\sigma}+\frac{\sigma_1^2}{n}$.
Moreover, applying the results in Corollary \ref{cor_SNR} gives the relation between $\theta$ and $\alpha^*$
for random graph signals.

\section{Performance Evaluation}
\label{sec_perform}
In this section we conduct experiments on synthetic graphs and real-world graph datasets to validate the developed bias-variance tradeoff analysis and the scaling behavior of the optimal regularization parameter $\alpha^*$ with respect to the parameter $\theta=\sqrt{\frac{1}{\textnormal{E-SNR}}}$. The graph signal $\bx^*$ is randomly drawn from a multivariate Gaussian distribution with $\mu=10$ and covariance  $\bs=\diag(\bone_n)$. The noise $\be$ is generated by a multivariate Gaussian distribution with zero mean and covariance $\bSigma=\diag(\bsigma)$, where $\bsigma=\sigma^2 \bone_n$. From Corollary \ref{cor_random_signal}, the E-SNR becomes $\frac{1}{\sigma^2}$ and hence $\theta=\sigma$. To investigate the scaling behavior of  $\alpha^*$ under different regimes of $\theta$, $\alpha^*$ 
is numerically obtained via grid search in the range $[0, b]$ with  $t$ uniform samples on the log-scale, where $b$ and $t$ are specified in each experiment. The results presented in this paper are averaged over 50 realizations.

We generate Erdos-Renyi random graphs with different node-pair connection probability $p$ to study the difference between $\textnormal{MSE}(\alpha)$ and $\textnormal{MSE-UB}(\alpha)$. Fig. \ref{Fig_ER_comp} shows the curves of per-node $\textnormal{MSE}(\alpha)$ and 
$\textnormal{MSE-UB}(\alpha)$ of three selected value $\alpha$ at different scales, where per-node $\textnormal{MSE}(\alpha)$ is the MSE divided by the number of nodes $n$. It is observed that the curves of $\textnormal{MSE}(\alpha)$ and 
$\textnormal{MSE-UB}(\alpha)$ have similar tendency with respect to $p$, and they collapse to the same value when $p=1$ (complete graphs), which justifies Corollary \ref{cor_MSE_bound}.

Fig. \ref{Fig_ER} displays the optimal regularization parameter  $\alpha^*$ obtained from minimizing MSE($\alpha$) and  MSE-UB($\alpha$) under different $\theta$, respectively, 
in Erdos-Renyi random graphs and in Watts-Strogatz small-world random graphs \cite{Watts98} with rewiring probability $q$ and average degree $d$. In the high E-SNR regime (small $\theta$), $\alpha^*$ is close to zero as proved in Corollary \ref{cor_SNR}. Furthermore, as one sweeps $\theta$, an abrupt boost in $\alpha^*$ followed by linear scaling with $\theta$ is observed, which is consistent with the analysis in Corollary \ref{cor_SNR}. Fixing $\theta$, we also observe that although $\alpha^*$ obtained from MSE($\alpha$) and $\textnormal{MSE-UB}(\alpha)$ are distinct, the corresponding curves of per-node MSE are nearly identical, especially in the large $\theta$ (low E-SNR) and small $\theta$  (high E-SNR) regimes, since $\textnormal{MSE-UB}(\alpha)$ is a tight upper envelope function of MSE($\alpha$) as stated in Corollary \ref{cor_MSE_bound}.

\begin{figure}[!t]
	\centering
	\includegraphics[width=2.6in]{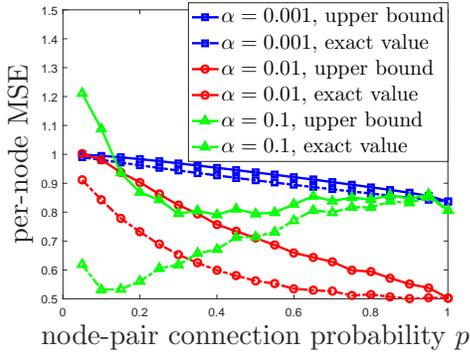}
			\vspace{-2mm}		
	\caption{Per-node $\textnormal{MSE}(\alpha)$ and $\textnormal{MSE-UB}(\alpha)$ in Erdos-Renyi random graphs with $n=100$ nodes. The curves of $\textnormal{MSE}(\alpha)$ and 
		$\textnormal{MSE-UB}(\alpha)$ collapse to the same value when $p=1$ (a complete graph), which justifies Corollary \ref{cor_MSE_bound}.}
	\label{Fig_ER_comp}
	\vspace{-4mm}
\end{figure}

	\begin{figure}[t]
		\centering
		\begin{subfigure}[b]{0.5\linewidth}
			\includegraphics[width=\textwidth]{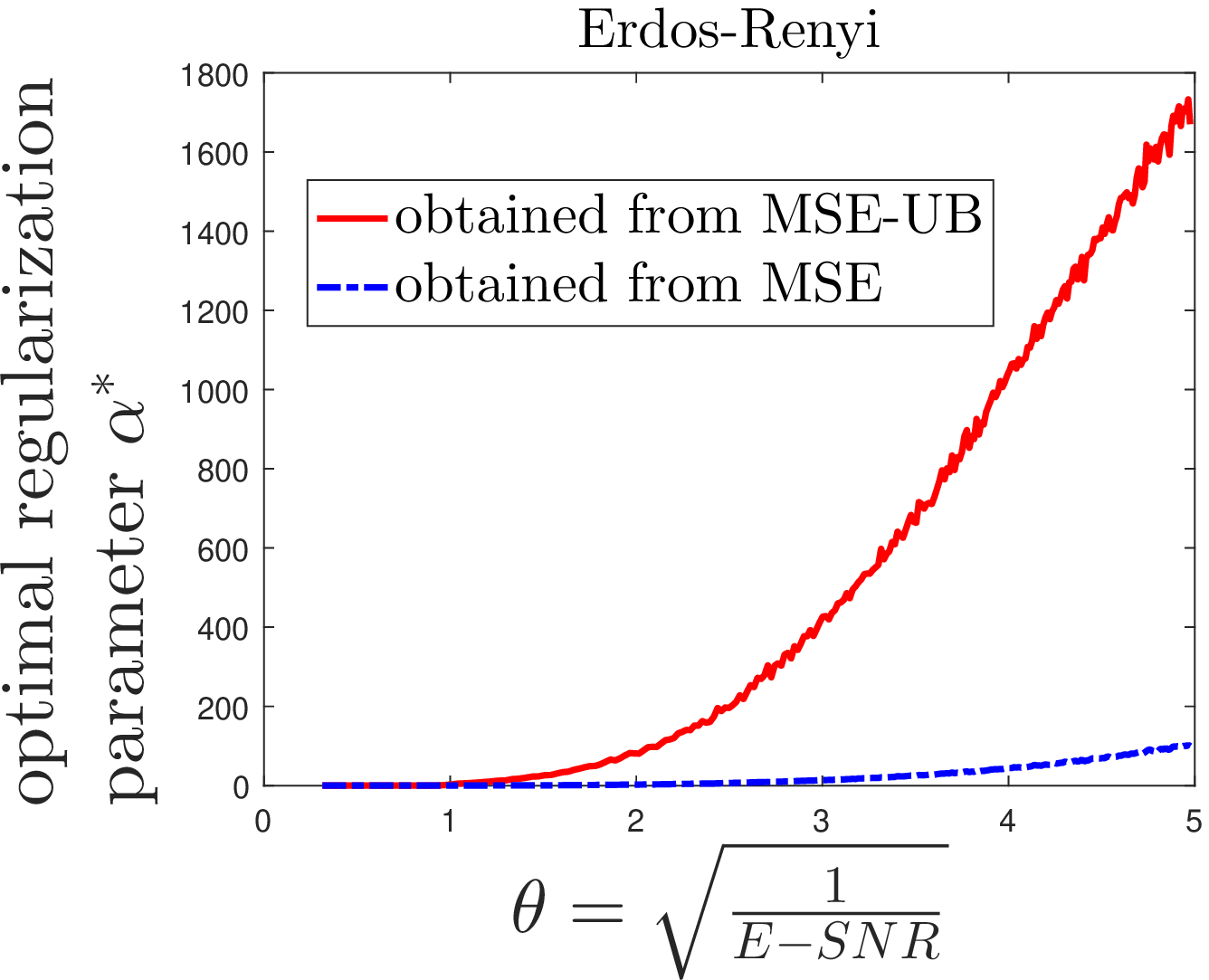}
		\end{subfigure}%
		\centering
		\begin{subfigure}[b]{0.5\linewidth}
			\includegraphics[width=\textwidth]{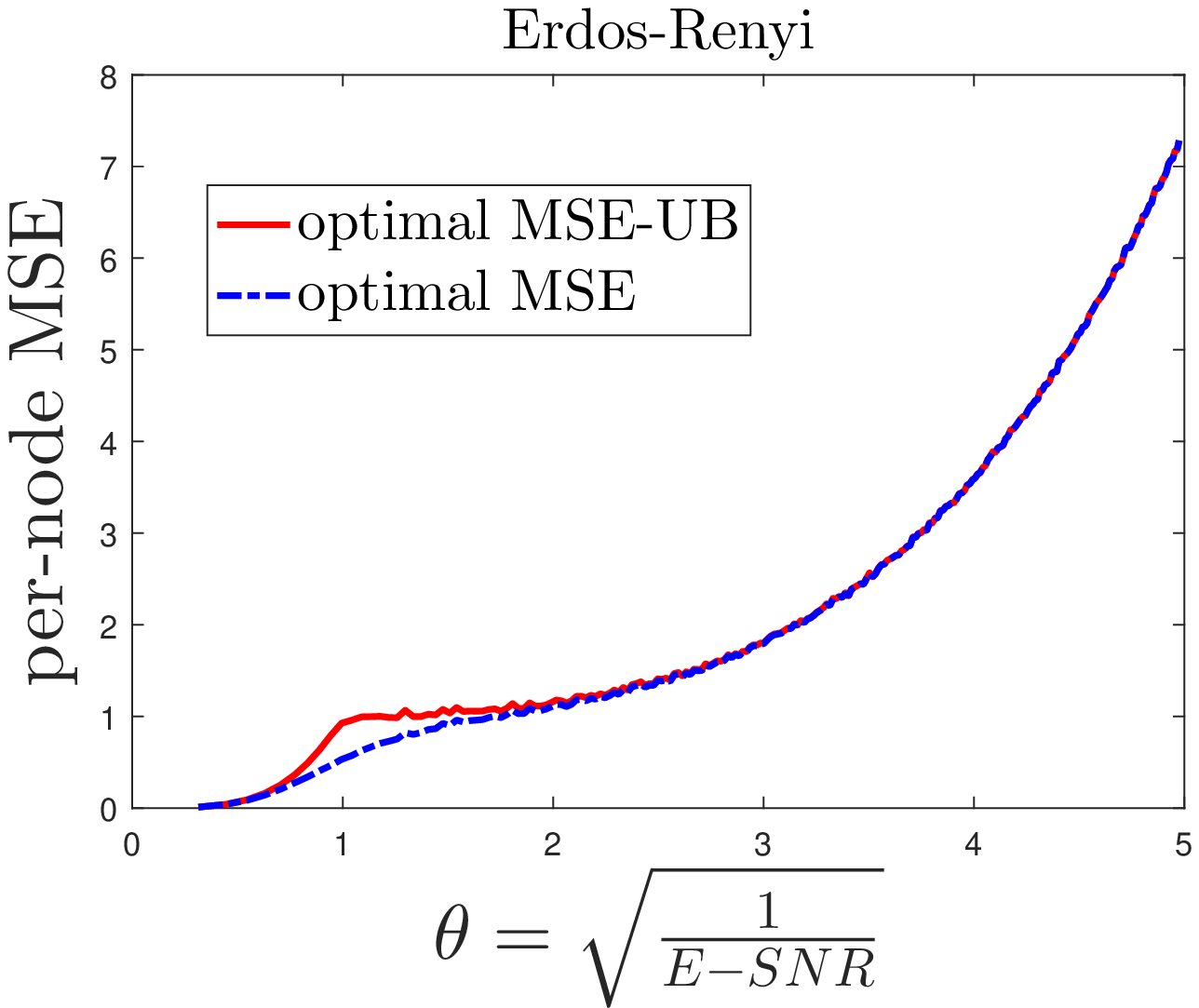}
		\end{subfigure}
		\\
		\begin{subfigure}[b]{0.5\linewidth}
			\includegraphics[width=\textwidth]{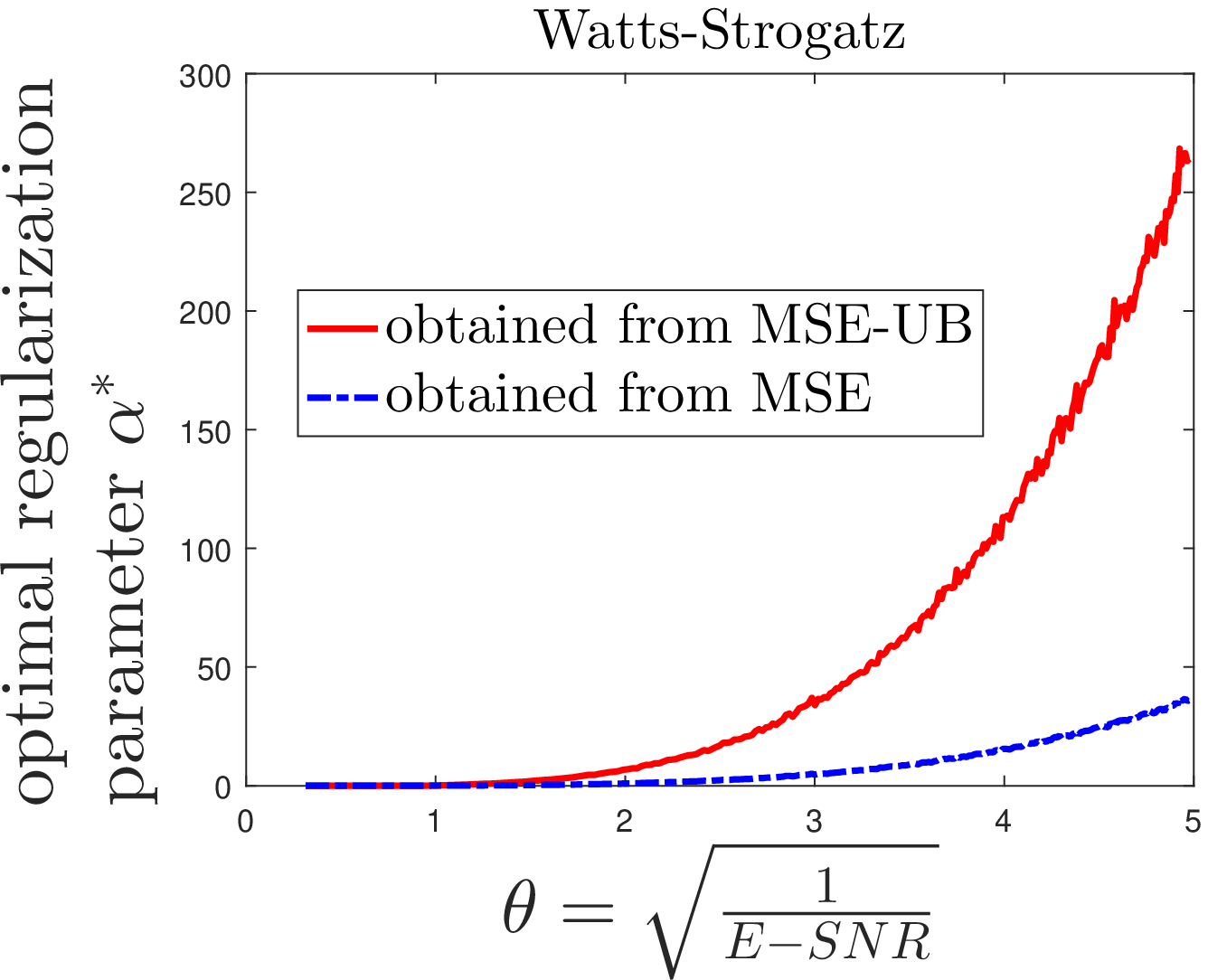}
		\end{subfigure}%
		\centering
		\begin{subfigure}[b]{0.5\linewidth}
			\includegraphics[width=\textwidth]{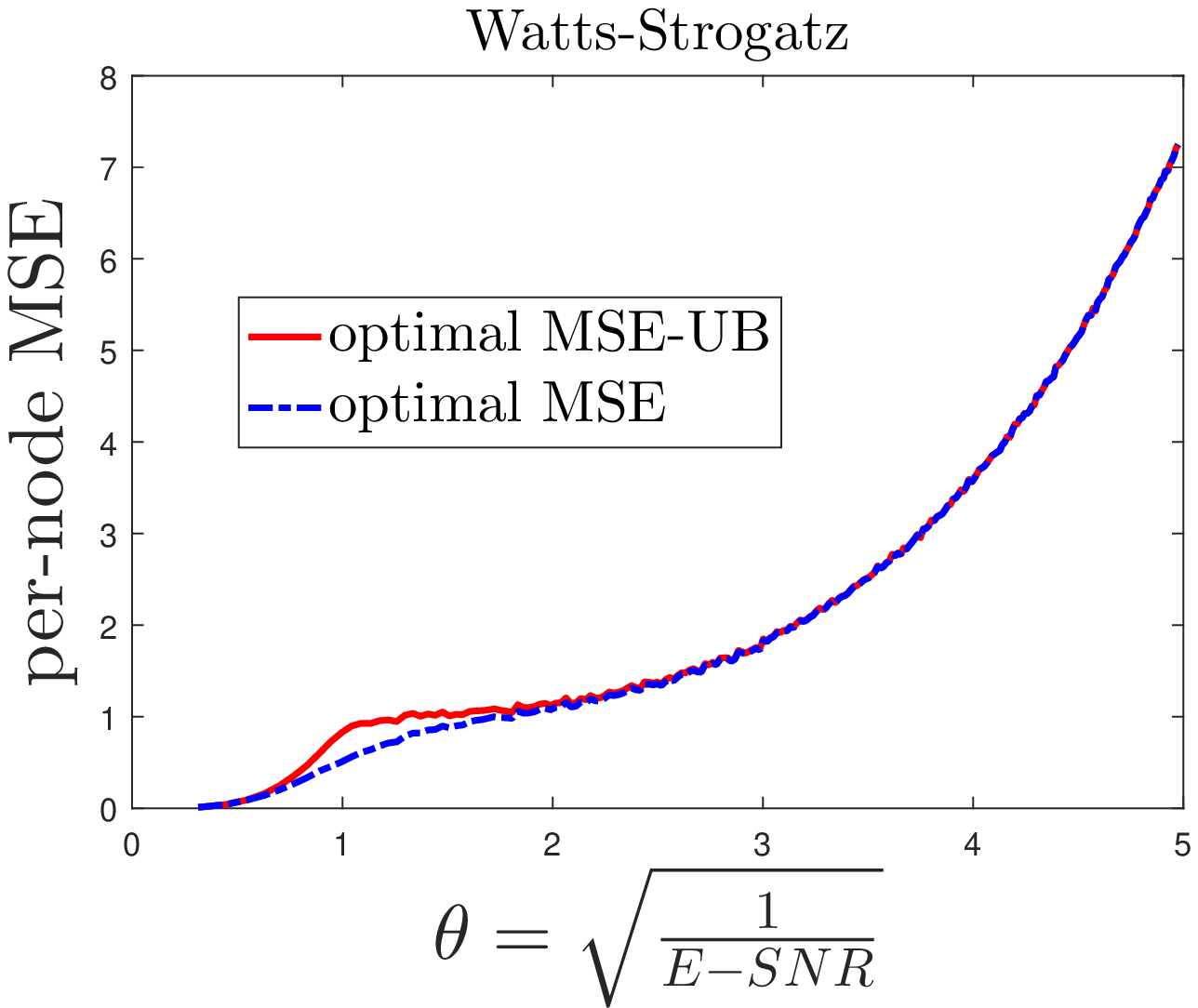}
		\end{subfigure}
				\vspace{-6mm}				
		\caption{Optimal regularization parameter $\alpha^*$ and the corresponding per-node MSE under different $\theta$ in Erdos-Renyi random graphs with $p=0.1$ and in Watts-Strogatz random graphs with $q=0.4$ and $d=20$. $n=100$, $b=2000$ and $t=10^4$. The scaling behavior of $\alpha^*$ validates Corollary \ref{cor_SNR}, and the corresponding per-node MSE  curves are nearly identical. }
		\label{Fig_ER}
			\vspace{-4mm}
	\end{figure}

	\begin{figure}[t]
		\centering
		\begin{subfigure}[b]{0.5\linewidth}
			\includegraphics[width=\textwidth]{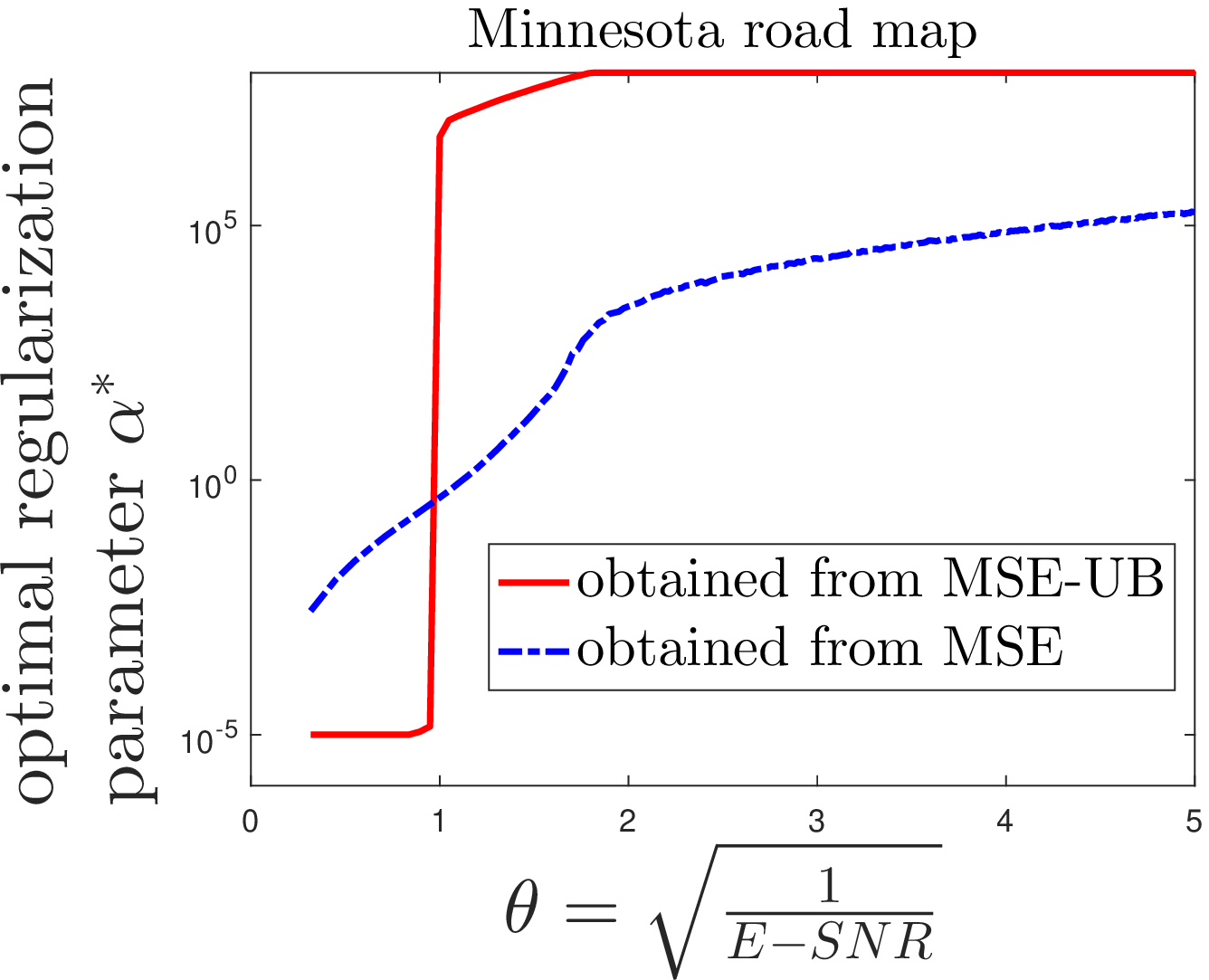}
		\end{subfigure}%
		\centering
		\begin{subfigure}[b]{0.5\linewidth}
			\includegraphics[width=\textwidth]{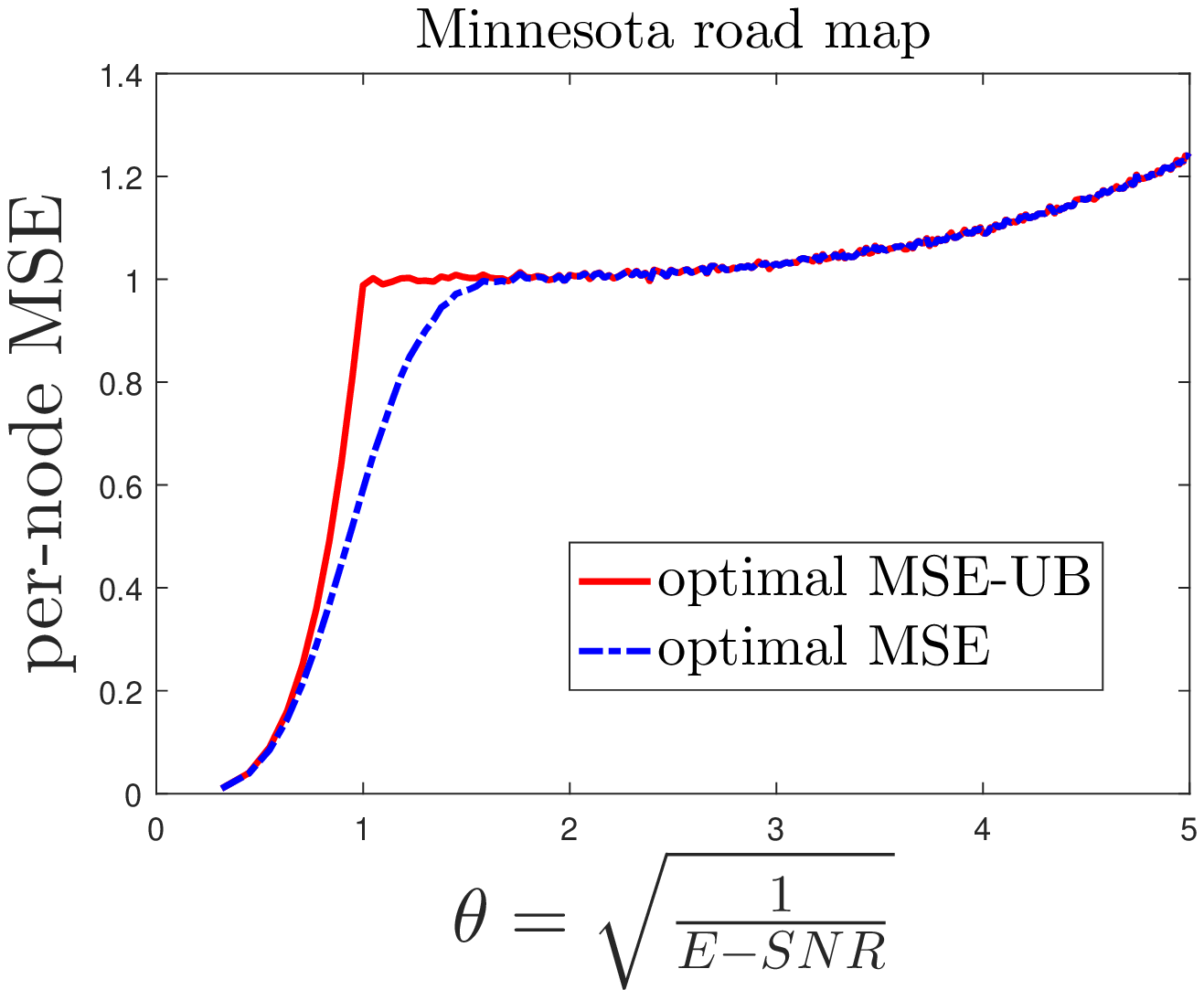}
		\end{subfigure}
		\\
		\begin{subfigure}[b]{0.5\linewidth}
			\includegraphics[width=\textwidth]{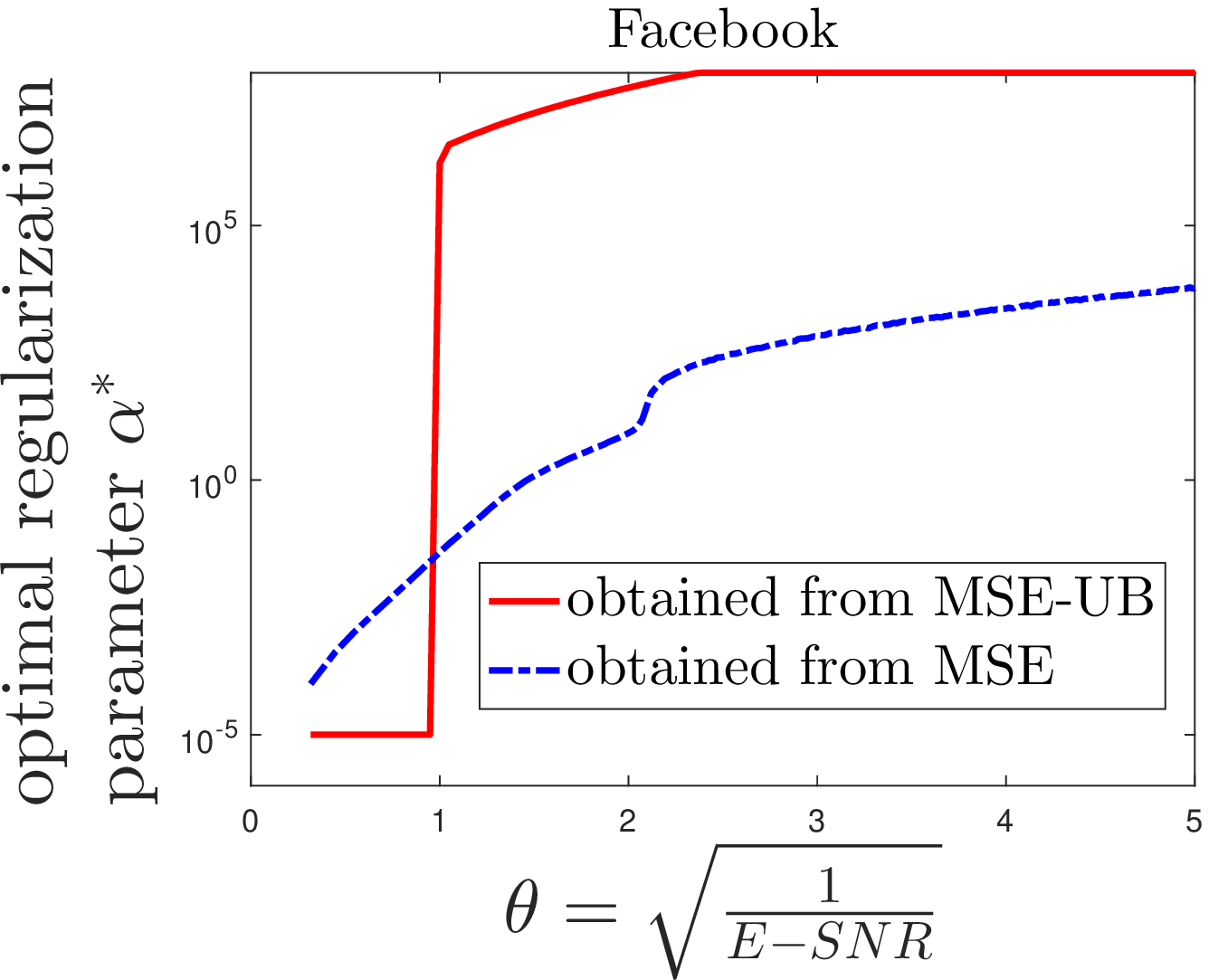}
		\end{subfigure}%
		\centering
		\begin{subfigure}[b]{0.5\linewidth}
			\includegraphics[width=\textwidth]{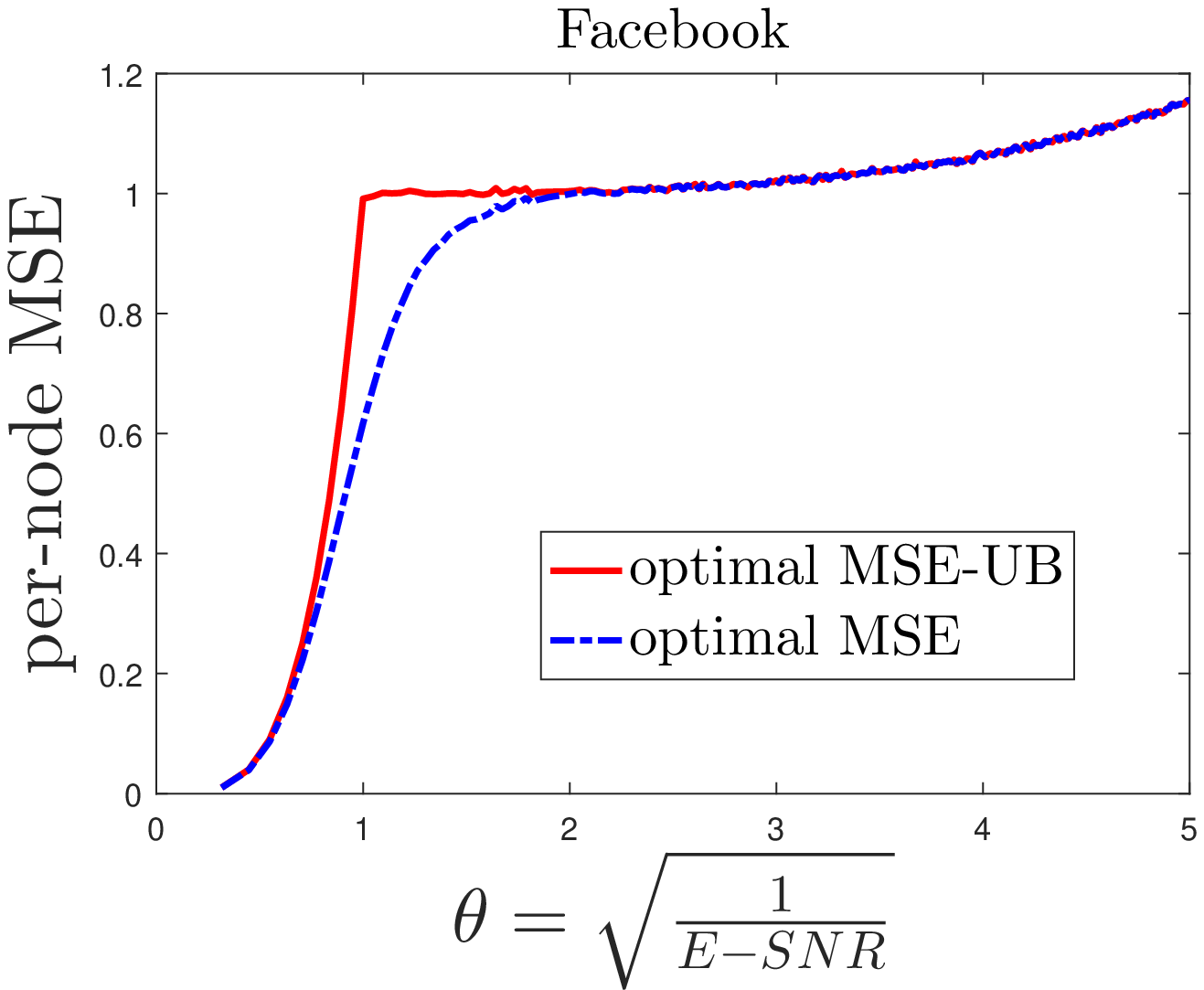}
		\end{subfigure}	
		\\
		\begin{subfigure}[b]{0.5\linewidth}
			\includegraphics[width=\textwidth]{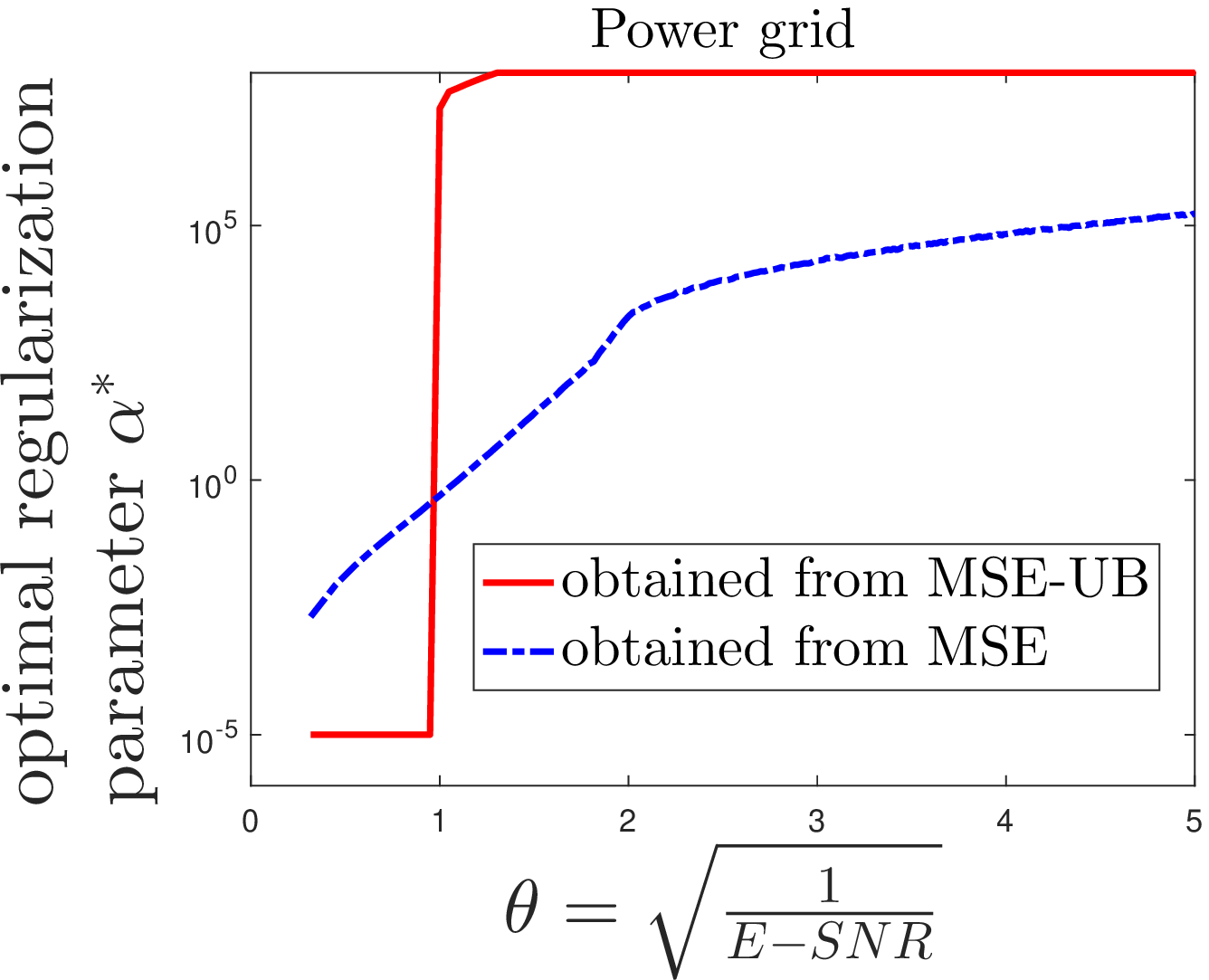}
		\end{subfigure}%
		\centering
		\begin{subfigure}[b]{0.5\linewidth}
			\includegraphics[width=\textwidth]{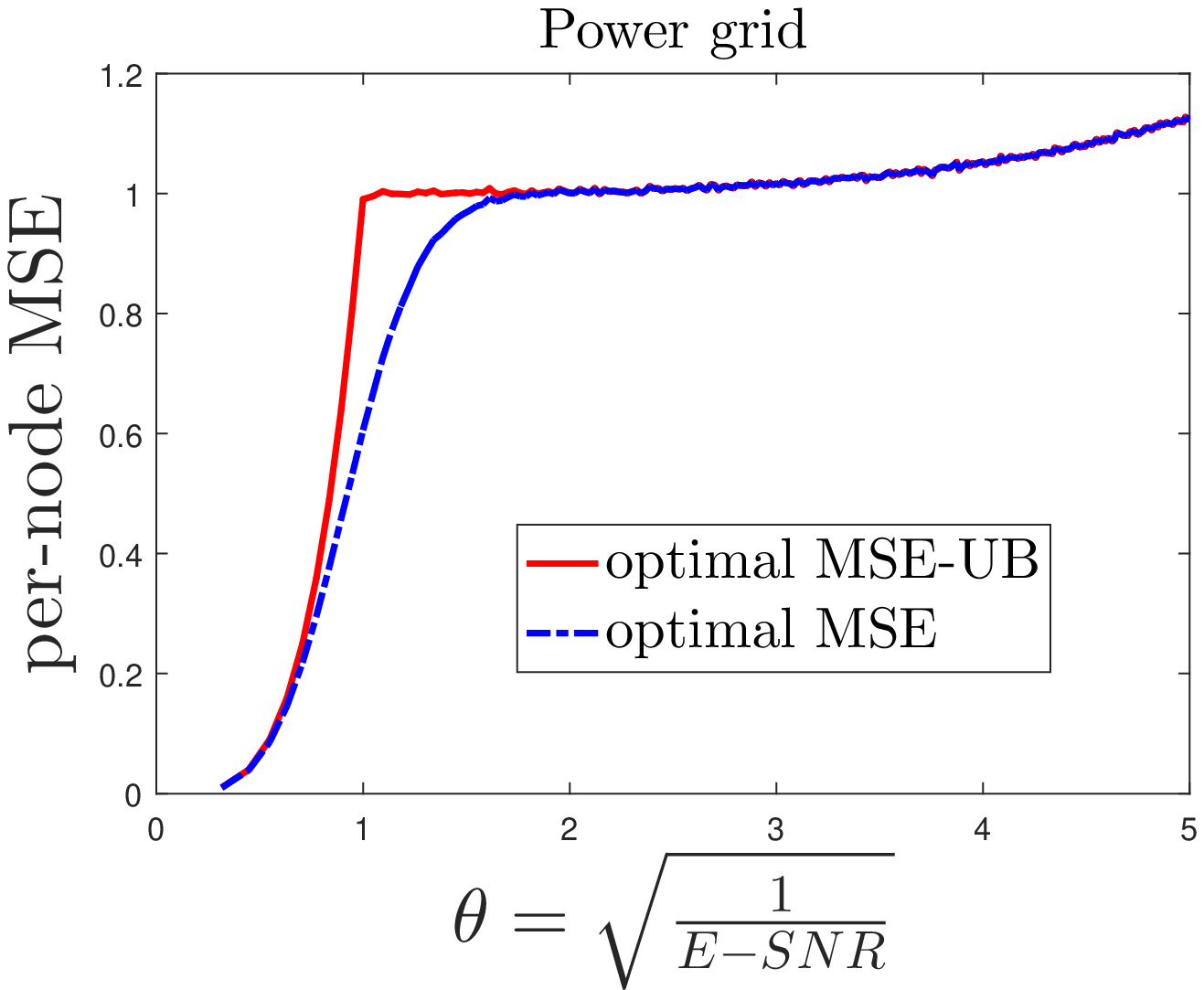}
		\end{subfigure}		
		\vspace{-6mm}		
		\caption{Optimal regularization parameter $\alpha^*$ (in log-scale) and the corresponding per-node MSE under different $\theta$ in three real-world graphs with $b=10^8$ and $t=10^3$. The saturation effect in $\alpha^*$ is due to the upper bound $b$ for grid search. Consistent with the analysis in  Corollary \ref{cor_SNR}, the results suggest that in most cases selecting a mediocre regularization parameter $\alpha$ is often suboptimal for minimizing the MSE. }
		\label{Fig_data}
		\vspace{-2mm}
	\end{figure}

Fig. \ref{Fig_data} displays the experimental results in three real-world graph datasets, including the Minnesota road map of 2640 nodes and 3302 edges {\cite{MATLAB_BGL}, the Facebook friendship graph of 4039 nodes and 88234 edges \cite{mcauley2012learning}, and the U.S. western power grid network of  4941 nodes and 6594 edges \cite{Watts98}. Consistent with the experimental results in synthetic graphs, similar scaling effect of $\alpha^*$ and near-optimal performance on per-node MSE are observed in real-world graph datasets. More importantly, as indicated in Corollary \ref{cor_SNR}, these experimental results suggest that  in most cases (i.e., different $\theta$) selecting a mediocre regularization parameter $\alpha$ is  often suboptimal for minimizing the MSE. Instead, assigning a large (small) regularization parameter in the large (small) $\theta$ regime is more effective in minimizing the MSE.

\section{Conclusion}
\label{sec_conclusion}
The contributions of this paper are twofold. First, we study the bias-variance tradeoff of graph Laplacian regularizer (GLR) and specify the scaling law of the optimal regularization parameter. We show that an abrupt boost in the optimal regularization parameter is expected when one sweeps a novel signal-to-noise ratio (SNR) parameter $\theta$, which suggests that selecting a mediocre regularization parameter is often suboptimal for minimizing the mean squared error.  Second, we apply the developed analysis to random, band-limited, and multiple-sampled graph signals and specify the corresponding SNR parameter $\theta$. Experimental results on synthetic and real-world graphs validate our analysis on the scaling effect of optimal regularization parameter, and demonstrate near-optimal performance in mean squared error, which provides new insights on signal processing and machine learning methods involving GLR. Future work includes extending the current framework to multi-stage bias-variance tradeoff with GLR.

\clearpage
\bibliographystyle{IEEEtran}
\bibliography{IEEEabrv20160824,CPY_ref_20170421}

\clearpage
\setcounter{equation}{0}
\setcounter{figure}{0}
\setcounter{table}{0}
\setcounter{page}{1}
\makeatletter
\renewcommand{\theequation}{S\arabic{equation}}
\renewcommand{\thefigure}{S\arabic{figure}}
\section*{{\LARGE Supplementary Material}}
\appendices
\section{}
\label{proof_var}
When $\alpha=0$, $\cov(\bxhat)=\bSigma$.
To show that \textnormal{Var($\alpha$)} $\leq$ \textnormal{Var($0$)} for any $\alpha>0$, from (\ref{eqn_variance}) it suffices to show that  $\trace(\bSigma)-\trace(\bH^2 \bSigma)=\trace[(\bI-\bH^2)\bSigma] \geq 0$  for any $\alpha>0$. For any $\alpha>0$, observe  from (\ref{eqn_H}) that the eigenvalue decomposition of $\bI-\bH^2$ is $\bI-\bH^2=\sum_{i=2}^n \Lb 1- \frac{1}{\lb 1+ \alpha \lambda_i \rb^2 } \Rb \bv_i \bv_i^T$, which means  $\bI-\bH^2$ is positive definite (PD) since $1+\alpha \lambda_i > 1$ for all $2 \leq i \leq n$. Finally, since $\Sigma$ is a covariance matrix and hence PSD, the term
 $\trace[(\bI-\bH^2)\bSigma] \geq 0$  ($\trace[(\bI-\bH^2)\bSigma]>0$ if $\Sigma$ has full rank) \cite{HornMatrixAnalysis}, which completes the proof. 

\section{}
\label{proof_MSE}
From (\ref{eqn_bias}), the squared bias is \textnormal{Bias($\alpha$)}$^2=\sum_{i=2}^n q_i^2 (\bv_i^T \bx^*)^2$. Recall that $\{\bv_i\}_{i=2}^n$ are eigenvectors of the graph Laplacian matrix $\bL$ such that $\bv_i^T  \bone_n=0$ for all $2 \leq i \leq n$. We have $(\bv_i^T \bx^*)^2=\Lb \bv_i^T (\bx^*- a\bone_n )\Rb^2$ for any $a \in \bbR$.
If $\Sigma=\diag(\bsigma)$, then the variance in (\ref{eqn_variance}) reduces to $\sum_{i=1}^n h_i^2 \sigma_i^2$. Finally, using (\ref{eqn_MSE}) and setting $a=\frac{\bone_n^T\bx^*}{n}$ give the results.

\section{}
\label{proof_MSE_bound}
Since $q_i=\frac{1}{1+\frac{1}{\alpha \lambda_i}} \leq \frac{1}{1+\frac{1}{\alpha \lambda_n}}$ and $h_i=\frac{1}{1+ \alpha \lambda_i} \leq \frac{1}{1+ \alpha \lambda_2} $, for all $2 \leq i \leq n$,    applying these results to Theorem \ref{thm_MSE}, we obtain the upper bound 	\textnormal{MSE-UB($\alpha$)} on 	\textnormal{MSE($\alpha$)}. If the graph $\cG$ is a complete graph of identical edge weight $w>0$, then $\lambda_i= w \cdot n $ for all $2 \leq i \leq n$. Therefore, the resulting \textnormal{MSE($\alpha$)} is identical to 	\textnormal{MSE-UB($\alpha$)}.

\section{}
\label{proof_match}
If the first two terms in the RHS of Corollary \ref{cor_MSE_bound} are order-matching, then there exists a constant $\beta > 0$ such that 
$\lb \frac{1}{1+\frac{1}{\alpha\lambda_n}} \rb^2= \beta^2 \cdot \lb \frac{1}{1+\alpha \lambda_2}\rb^2  \theta^2 $. Solving this equation gives 
\begin{align}
\alpha^* = \frac{( \beta \theta-1) \lambda_n +\sqrt{( \beta \theta-1)^2 \lambda_n^2+4 \lambda_n \lambda_2 \beta \theta}} {2 \lambda_n \lambda_2}. 
\end{align}

\section{}
\label{proof_SNR}
For the high E-SNR and low E-SNR regimes, the optimal order of $\alpha^*$ can be obtained by the Newton's generalized binomial expansion (binomial series expansion) that for any real $x$ such that $|x|<1$, $\sqrt{1+x}=1+\frac{x}{2}+O(x^2)$. Specifically, for the high E-SNR and low E-SNR regimes the term $\sqrt{( \beta \theta-1)^2 \lambda_n^2+4 \lambda_n \lambda_2 \beta \theta}$ in Theorem \ref{thm_match} can be approximated by $\sqrt{( \beta \theta-1)^2 \lambda_n^2+4 \lambda_n \lambda_2 \beta \theta} = | \beta \theta-1| \lambda_n \sqrt{ 1 +\frac{4 \lambda_n \lambda_2 \beta \theta}{( \beta \theta-1)^2 \lambda_n^2}} \approx | \beta \theta-1| \lambda_n \lb 1+ \frac{2 \lambda_n \lambda_2 \beta \theta}{( \beta \theta-1)^2 \lambda_n^2} \rb $. If $\beta \theta \gg 1$, then $ | \beta \theta-1| \lambda_n \lb 1+ \frac{2 \lambda_n \lambda_2 \beta \theta}{( \beta \theta-1)^2 \lambda_n^2} \rb \approx (\beta \theta -1) \lambda_n$, which implies $\alpha^*=O\lb \frac{\theta}{\lambda_2}\rb$. If $\beta \theta \ll 1$ (i.e., $(\beta \theta -1)^2 \approx 1$), then $ | \beta \theta-1| \lambda_n \lb 1+ \frac{2 \lambda_n \lambda_2 \beta \theta}{( \beta \theta-1)^2 \lambda_n^2} \rb \approx (1- \beta \theta) \lambda_n \lb 1+ \frac{2 \lambda_2 \beta \theta}{ \lambda_n} \rb$, which implies $\alpha^*=O\lb \frac{\theta}{\lambda_n}\rb$.
 For the moderate E-SNR regime,  the optimal order can be obtained from Theorem \ref{thm_match} using  the fact that $\beta \theta-1 \approx 0$. 

\section{}
\label{proof_multi_signal}
Let $\bxhat=\bH \by$ and let $\bxt=\bH \bybar$. It is easy to verity that $\bbE [\bxt]=\bbE [\bxhat]=\bH \bx^*$ and $\cov(\bxt)=\frac{\cov(\bxhat)}{T}$. Therefore, the bias of $\bxt$ is the same as in (\ref{eqn_bias}) and the variance of $\bxt$ is $\frac{\textnormal{Var$(\alpha)$}}{T}$, where Var$(\alpha)$ denotes the variance of $\bxhat$ as in (\ref{eqn_variance}). Finally, the results are obtained by following the same proof procedure as in Appendix \ref{proof_MSE_bound}.

\section{}	
\label{proof_band_signal}
Since $\bx^*=\sum_{j \in \cA} \omega_j\bv_j $,  by the orthogonality of eigenvectors, applying 
$\sum_{i=2}^n (\bv_i^T \bx^*)^2= \sum_{j \in \cA / \{1\}} \omega_j^2$ to Corollary \ref{cor_MSE_bound} completes the proof.

\section{}
Since $\bx^* \sim \cN(\mu \bone_n,\diag(\bs))$,  using the smoothing property of conditional expectation and following the same proof procedure as in Appendix \ref{proof_MSE} by setting $a=\mu$, and using Corollary \ref{cor_MSE_bound} gives 
\label{proof_random_signal}
\begin{align}
\label{eqn_pf1}
\textnormal{MSE-UB($\alpha$)}	
&= \lb \frac{1}{1+\frac{1}{\alpha\lambda_n}} \rb^2  \trace \lb \diag(\bs) (\bV \bV^T-\bone \bone^T) \rb  \nonumber \\
&~~~+ \lb \frac{1}{1+\alpha \lambda_2}\rb^2 (n-1) \overline{\sigma}+\sigma_1^2,
\end{align} 
where $\bV=[\bv_1~\bv_2~\cdots~\bv_n]$.
Applying the Von Neumann's trace inequality \cite{HornMatrixAnalysis},	we have 
\begin{align}
\label{eqn_pf2}
\trace \lb \diag(\bs) \bV \bV^T \rb	\leq \sum_{i=1}^n s_i^2=n \overline{s}, 
\end{align}
where the equality holds if $s_i=s \geq 0$ for all $1 \leq i \leq n$.
Finally, since 	 $\trace \lb \diag(\bs) \bone \bone^T\rb = \overline{s}$, applying  (\ref{eqn_pf2}) to  (\ref{eqn_pf1}) completes the proof.

\end{document}

%% file: notation_20161122.tex
\newcommand{\Lb}{\left[}
\newcommand{\Rb}{\right]}
\newcommand{\lb}{\left(}
\newcommand{\rb}{\right)}
\newcommand{\trace}{\textnormal{trace}}
\newcommand{\diag}{\textnormal{diag}}
\newcommand{\cov}{\textnormal{cov}}

\newcommand{\bone}{\mathbf{1}}

\newcommand{\bx}{\mathbf{x}}
\newcommand{\by}{\mathbf{y}}
\newcommand{\bybar}{\overline{\mathbf{y}}}

\newcommand{\bs}{\mathbf{s}}
\newcommand{\bsigma}{\boldsymbol{\sigma}}

\newcommand{\bv}{\mathbf{v}}

\newcommand{\be}{\mathbf{e}}

\newcommand{\bxt}{\widetilde{\mathbf{x}}}

\newcommand{\bxhat}{\widehat{{\mathbf{x}}}}

\newcommand{\bH}{\mathbf{H}}
\newcommand{\bQ}{\mathbf{Q}}
\newcommand{\bI}{\mathbf{I}}

\newcommand{\bSigma}{\mathbf{\Sigma}}

\newcommand{\bV}{\mathbf{V}}
\newcommand{\bW}{\mathbf{W}}

\newcommand{\bL}{\mathbf{L}}

\newcommand{\bR}{\mathbf{R}}
\newcommand{\bS}{\mathbf{S}}

\newcommand{\cA}{\mathcal{A}}

\newcommand{\cN}{\mathcal{N}}
\newcommand{\cG}{\mathcal{G}}
\newcommand{\cV}{\mathcal{V}}
\newcommand{\cE}{\mathcal{E}}

\newcommand{\bbR}{\mathbb{R}}
\newcommand{\bbE}{\mathbb{E}}